\pdfoutput=1
\documentclass[11pt]{article} 
\usepackage{times}
\usepackage[letterpaper, left=1in, right=1in, top=1in, bottom=1in]{geometry}

\usepackage[utf8]{inputenc} 
\usepackage[T1]{fontenc}    
\usepackage{url}            
\usepackage{booktabs}       
\usepackage{amsfonts}       
\usepackage{nicefrac}       
\usepackage{microtype}      
\usepackage{mkolar_definitions}
\usepackage{xspace}
\usepackage{amsmath}
\usepackage{algorithm,algorithmic}
\usepackage{color}
\usepackage{enumitem}
\usepackage{comment}
\usepackage{bm}
\usepackage{graphicx}

\usepackage{apptools}
\usepackage[page, header]{appendix}
\usepackage{titletoc}

\usepackage[scaled=.9]{helvet}

\usepackage{url}
\usepackage{authblk}
\usepackage{caption}

\usepackage[dvipsnames]{xcolor}
\usepackage[colorlinks=true, linkcolor=blue, citecolor=blue]{hyperref}

\usepackage{natbib}
\bibpunct{(}{)}{;}{a}{,}{,}

\renewcommand{\cite}{\citep}

\newcount\Comments  
\definecolor{darkgreen}{rgb}{0,0.5,0}
\definecolor{darkred}{rgb}{0.7,0,0}
\definecolor{teal}{rgb}{0.3,0.8,0.8}
\definecolor{orange}{rgb}{1.0,0.5,0.0}
\definecolor{purple}{rgb}{0.8,0.0,0.8}
\newcommand{\kibitz}[2]{\ifnum\Comments=1{\textcolor{#1}{\textsf{\footnotesize #2}}}\fi}

\usepackage{authblk}

\title{Contrastive estimation reveals topic posterior information \\ to linear models}
\author[1]{Christopher Tosh\thanks{c.tosh@columbia.edu}}
\author[2]{Akshay Krishnamurthy\thanks{akshaykr@microsoft.com}}
\author[1]{Daniel Hsu\thanks{djhsu@cs.columbia.edu}}
\affil[1]{Columbia University, New York, NY}
\affil[2]{Microsoft Research, New York, NY}

\begin{document}

\maketitle

\vspace{-2em}
\begin{abstract}
Contrastive learning is an approach to representation learning that utilizes naturally occurring similar and dissimilar pairs of data points to find useful embeddings of data. In the context of document classification under topic modeling assumptions, we prove that contrastive learning is capable of recovering a representation of documents that reveals their underlying topic posterior information to linear models. We apply this procedure in a semi-supervised setup and demonstrate empirically that linear classifiers with these representations perform well in document classification tasks with very few training examples.
\end{abstract}

\section{Introduction}
\label{sec:intro}
Using unlabeled data to find useful embeddings is a central challenge in the field of representation learning. Classical approaches to this task often start by fitting some type of structure to the unlabeled data, such as a generative model or a dictionary, and then embed future data by performing inference using the fitted structure~\cite{BNJ03, RBLPN07}. While this approach has sometimes enjoyed good empirical performance, it is not without its drawbacks. One issue is that learning structures and performing inference is often hard in general~\cite{SR11, AGM12}. Another issue is that we must a priori choose a structure and method for fitting the unlabeled data, and unsupervised methods for learning these structures can be sensitive to model misspecification~\cite{KRS14}.

{Contrastive learning} (also called noise contrastive estimation, or NCE) is an alternative approach to representation learning that tries to capture the latent structure in unlabeled data implicitly. Informally, contrastive learning methods formulate a classification problem in which the goal is to distinguish examples that naturally occur in pairs, called positive samples, from randomly paired examples, called negative samples. The particular choice of positive samples depends on the setting. In image representation problems, for example, neighboring frames from videos may serve as positive examples~\cite{WG15}. In text modeling, the positive samples may be neighboring sentences~\cite{LL18, DCLT18}. The idea is that in the course of learning to distinguish between semantically similar positive examples and randomly chosen negative examples, the representations constructed along the way will capture some of that latent semantic information.

In this work, we consider contrastive learning for document modeling where we have a corpus of text documents and our goal is to construct a useful vector representation for these documents. In this setting, there is a natural source of positive and negative examples: a positive example is simply a document from the corpus, and a negative example is one formed by pasting together the first half of one document and the second half of another document. We prove that when the corpus is generated by a topic model, learning to distinguish between these two types of documents yields representations that are closely related to their underlying latent variables. In fact, we show that linear functions of these representations can approximate the posterior mean of any continuous function of the latent variables.

One potential application of contrastive learning is in a semi-supervised setting, where there is a small amount of labeled data as well as a much larger collection of unlabeled data. In these situations, purely supervised methods that fit complicated models may have poor performance due to the limited amount of labeled data. On the other hand, when the labels are well-approximated by some function of the latent structure, our results show that an effective strategy is to fit linear functions, which may be learned with relatively little labeled data, on top of contrastive representations. In our experiments, we verify empirically that this approach produces reasonable results.

\subsection{Related work}

There has been much work on reducing unsupervised problems to synthetically-generated supervised problems. In dynamical systems modeling, \citet{LSZ09} showed that if one can solve a few forward prediction problems, then it is possible to track the underlying state of a  nonlinear dynamical system. In anomaly/outlier detection, a useful technique is to learn a classifier that distinguishes between true samples from a distribution and fake samples from some synthetic distribution~\cite{SHS05, AZL06}. Similarly, estimating the parameters of a probabilistic model can be reduced to learning to classify between true data points and randomly generated points~\cite{GH10}.

In the context of natural language processing, methods such as skip-gram and continuous bag-of-words turn the problem of finding word embeddings into a prediction problem~\cite{MCCD13,MSCCD13}. Modern language representation training algorithms such as BERT and QT  also use naturally occurring classification tasks such as predicting randomly masked elements of a sentence or discriminating whether or not two sentences are adjacent~\cite{DCLT18, LL18}. Training these models often employs a technique called negative sampling, in which softmax prediction probabilities are estimated by randomly sampling examples; this bears close resemblance to the way that negative examples are produced in contrastive learning.

Most relevant to the current paper, \citet{AKKPS19} gave a theoretical analysis of contrastive learning. They considered the specific setting of trying to minimize the contrastive loss
\[ L(f) \ = \ \EE_{x, x_+, x_-} [\ell \left( f^\top(x) (f(x_+) - f(x_-) ) \right)] \]
where $(x, x_+)$ is a positive pair and $(x,x_-)$ is a negative pair. They showed that if there is an underlying collection of latent classes and positive examples are generated by draws from the same class, then minimizing the contrastive loss over embedding functions $f$ yields good representations for the classification task of distinguishing latent classes.

The main difference between our work and that of~\citet{AKKPS19} is
that we adopt a generative modeling perspective and induce the
contrastive distribution naturally, while they do not make generative
assumptions but assume the contrastive distribution is directly
induced by the downstream classification task. In particular, our
contrastive distribution and supervised learning problem are only
\emph{indirectly} related through the latent variables in the
generative model, while~\citeauthor{AKKPS19} assume an explicit connection.
The focus of our work is therefore complementary to theirs: we study the
types of functions that can be succinctly expressed with the
contrastive representation in our generative modeling setup. In
addition, our results apply to semi-supervised regression, but it is
unclear how to define their contrastive distribution in this setting;
this makes it difficult to apply their results here.

\subsection{Overview of results}

In Section~\ref{sec:algorithm}, we present a simple contrastive learning procedure that is based on learning a function to determine if two bag-of-words vectors were generated by randomly partitioning a document or if they came from two different documents. We also present a way to turn the outputs of such a function into an embedding of future documents.

In Section~\ref{sec:topic-structure}, we show that under certain topic modeling assumptions, the document embeddings we construct from contrastive learning capture underlying topic structure. In particular, we demonstrate that linear functions of these embeddings are capable of representing any polynomial of the topic posterior vector.

In Section~\ref{sec:errors}, we analyze the errors that arise in the finite sample setting. We show that whenever we can achieve low prediction error on the contrastive learning task, linear functions learned on the resulting representations must also be high quality.

In Section~\ref{sec:experiments}, we apply our contrastive learning procedure to a semi-supervised document classification task. We show that these embeddings outperform several natural baselines, particularly in the low labeled data regime. We also investigate the effect of contrastive model capacity and model performance on the contrastive task on embedding quality.

In Section~\ref{sec:simulations}, we investigate the effects of model capacity and corpus size on a simulated topic recovery task. We demonstrate that increasing either of these quantities leads to an improvement in topic recovery accuracy.
\section{Setup}
\label{sec:setup}

Let $\Vcal$ denote a finite vocabulary, and take $\Kcal$ to be a finite 
set of $K$ topics. We consider a very
general topic modeling setup, which generates documents according to
the following process. First, a topic distribution $\wb \in
\Delta(\Kcal)$ is drawn, and then each of $m$ words $x_1,\ldots,x_m$
are drawn by sampling $z_i \sim \wb$ and then $x_i \sim O(\cdot |
z_i)$. The parameters of this model that are of primary interest
are the topic distributions $O(\cdot | k) \in \Delta(\RR^d)$. 
Note that documents need not have the same number of words.

This model is quite general and captures topic models
such as Latent Dirichlet Allocation (LDA) as well as topic models with word
embeddings. 
In LDA, the topic distributions $O(\cdot \mid k)$ are unconstrained.
When word embeddings are introduced, we set $O(\cdot \mid k) =
\textrm{softmax}(A\beta_k)$ where $A \in \RR^{|\Vcal|\times L}$ is a
latent embeddings matrix and $\beta_1,\ldots,\beta_k \in \RR^L$ are
latent ``context vectors.''  





We assume that there is a joint distribution $\Dcal$ supported on
triples $(\xb,\wb,\ell)$ where $\xb$ is a document, $\wb$ is the topic
distribution and $\ell$ is a label.
Triples are generated by first sampling $(\wb,\xb)$ from the
above topic model, and then sampling $\ell$ from some conditional
distribution that depends on the topics $\wb$, denoted $\Dcal(\cdot
\mid \wb)$. Our goal is to characterize the functional forms of
conditional distribution that are most suited to contrastive learning.

In the semi-supervised setting, we are given a collection $\Ucal
\defeq \{\xb_1,\ldots,\xb_{n_U}\}$ of unlabeled documents sampled from
the marginal distribution $\Dcal_{\xb}$, where topics and labels are
suppressed. We also have access to $n_L \ll n_U$ \emph{labeled}
samples $\Lcal \defeq \{(\xb_1,\ell_1),\ldots,(\xb_{n_L},\ell_{n_L}\}$
sampled from the distribution $\Dcal_{\xb,\ell}$, where only the topics are
suppressed. In both datasets, we never observe any topic distributions
$\wb$. From this data, we would like to learn a predictor 
$f: \xb \mapsto \hat{\ell}$ that predicts the label given the document.

\section{Contrastive learning algorithm}
\label{sec:algorithm}

\begin{algorithm}[t]
\caption{Contrastive Estimation with Documents}
\label{alg:main}
\begin{algorithmic}
\STATE \textbf{Input:} Corpus $\Ucal = \{\xb_i\}$ of documents.
\STATE $S = \emptyset$ 
\FOR{$i=1,\ldots,n$}
\STATE Sample $\xb_1,\xb_2 \sim \textrm{unif}(\Ucal)$. Split $\xb_i = (\xb_i^{(1)},\xb_i^{(2)})$. 
\begin{align*}
S \gets S \cup \left\{ \begin{aligned}
\{(\xb_1^{(1)},\xb_1^{(2)},1)\} & \textrm{ w.p } \nicefrac{1}{2} \\
\{(\xb_1^{(1)},\xb_2^{(2)},0)\} & \textrm{ w.p } \nicefrac{1}{2}
\end{aligned}\right.
\end{align*}
\ENDFOR
\STATE Learn $\hat{f} \gets \argmin_{f \in \Fcal} \sum_{S} (f(\xb^{(1)},\xb^{(2)}) - y)^2$
\STATE Select landmarks documents $\lb_1,\ldots,\lb_M$ and embed
\begin{align*}
\hat{\phi}(\xb) = \rbr{ \frac{\hat{f}(\xb,\lb_i )}{1-\hat{f}(\xb,\lb_i)} : i \in [M]}.
\end{align*}
\end{algorithmic}
\end{algorithm}

In contrastive learning, examples come in the form of similar and
dissimilar pairs of points, where the exact definition of
similar/dissimilar depends on the task at hand. Our construction of
similar pairs will take the form of randomly splitting a document into
two documents, and our dissimilar pairs will consist of subsampled
documents from two randomly chosen documents. In the generative
modeling setup, since the words are i.i.d.~conditional on the topic
distribution, a natural way to split a document $x$ into two is to
simply call the first half of the words $x^{(1)}$ and the second half
$x^{(2)}$. In our experiments, we split the documents randomly.

The contrastive representation learning procedure is displayed in
Algorithm~\ref{alg:main}. It utilizes a finite-sample approximation to
the following contrastive distribution.

\begin{itemize}
	\item Sample a document $x$ and partition it into $(x^{(1)}, x^{(2)})$. 
	Alternatively, we may think of our documents as coming `pre-partitioned,' and denote the
	marginal distributions of $x^{(1)}$ and $x^{(2)}$ as $\mu_1$ and $\mu_2$, respectively.
	\item With probability 1/2, output $(x^{(1)}, x^{(2)}, 1)$.
	\item With probability 1/2, sample a second document $(\tilde{x}^{(1)}, \tilde{x}^{(2)})$ 
	and output  $(x^{(1)}, \tilde{x}^{(2)}, 0)$.
\end{itemize}

We denote the above distribution over $(x,x',y)$ as $\Dcal_{c}$, and
we frame the contrastive learning objective as a least squares problem
between positive and negative examples.
\begin{equation}
\label{eqn: contrastive objective}
 \mini_f \, \EE_{(x, x', y) \sim \Dcal_{c}}\left [ \left( f(x,x') - y \right)^2\right ] 
\end{equation}

In our algorithm, we approximate this expectation via sampling and
optimize the empirical objective, which yields an approximate
minimizer $\hat{f}$ (chosen from some function class $\Fcal$). We use
$\hat{f}$ to form document representations by concatenating
predictions on a set of \emph{landmark documents}. Formally, we select
documents $l_1,\ldots,l_M$ and represent document $x$ via the mapping:
\begin{align*}
\hat{\phi}: x \mapsto \rbr{\frac{\hat{f}(x,l_i)}{1-\hat{f}(x,l_i)} : i \in [M]}.
\end{align*}
This yields the final document-level representation, which we use for
downstream tasks. 

For our analysis, let $f^\star$ denote the Bayes optimal predictor, or
the global minimizer, for Eq.~\eqref{eqn: contrastive objective}.
By Bayes' theorem we have that $g^\star := f^\star/(1-f^\star)$ satisfies the following
\begin{align*}
g^\star(x, x')  &:= \frac{ f^\star(x, x')}{1-  f^\star(x, x')} = \frac{\PP \left( y = 1 \, | \, x,  x' \right)}{ \PP \left( y = 0 \, | \, x,  x' \right) } \\
  &= \frac{\PP \left(x^{(1)} = x, x^{(2)} = x' \right) }{ \PP \left(x^{(1)} = x \right) \PP \left(x^{(2)} = x' \right)}.
\end{align*}
Letting $l_1, \ldots, l_M$ denote $M$ fixed documents, the \emph{oracle representation} of a document $x$ is
\begin{equation}
\label{eqn: landmark embedding}
g^\star(x, l_{1:M}) := (g^\star(x, l_1), \ldots, g^\star(x, l_M)).
\end{equation}
This representation takes the same form as $\hat{\phi}$ except that
the we have replaced the learned predictor $\hat{f}$ with the Bayes
optimal one $f^\star$.\footnote{Strictly speaking, we should first partition $x=(x^{(1)}, x^{(2)})$, only use landmarks that occur as second-halves of documents, and embed $x \rightarrow g^\star(x^{(1)}, l_{1:M})$. For the sake of clarity, we will ignore this technical issue here and in the remainder of the paper.}

\section{Recovering topic structure}
\label{sec:topic-structure}

In this section, we focus on expressivity of the contrastive
representation, showing that polynomial functions of the topic
posterior can be represented as \emph{linear} functions of the
representation. To do so, we ignore statistical issues and assume that
we have access to the oracle representations $g^\star(x,\cdot)$.
In the next
section we address statistical issues.


Recall the generative topic model process for a document $x$.
\begin{itemize}
	\item Draw a topic vector $w \in \Delta(\Kcal)$.
	\item For $i = 1, \ldots, \text{length}(x)$:
	\begin{itemize}
		\item Draw $z _i \sim \text{Categorical}(w)$.
		\item Draw $x_i \sim O(\cdot | z_i)$. 
	\end{itemize}
\end{itemize}

We will show that when documents are generated according to the above
model, the embedding of a document $x$ in Eq.~\eqref{eqn:
  landmark embedding} is closely related its underlying topic vector $w$.

\subsection{The single topic case}
To build intuition for the embedding in Eq.~\eqref{eqn: landmark
  embedding}, 
we first consider the case where each document's probability vector
$w$ is supported on a single topic, i.e., $w \in \{ e_1, \ldots,
e_K\}$ where $e_i$ is the $i^{\textrm{th}}$ standard basis
element. Then we have the following lemma.

\begin{lemma}
\label{lem: single topic g}
For any documents $x, x'$, 
\[ g^\star(x,x') = \frac{\eta(x)^\top \psi(x')}{\PP(x^{(2)} = x')}, \]
where $\eta(x)_k := \PP(w = e_k | x^{(1)} = x)$ is the topic
posterior distribution and $\psi(x)_k := \PP(x^{(2)} = x | w = e_k)$
is the likelihood.
\end{lemma}
\begin{proof}
Conditioned on the topic vector $w$, $x^{(1)}$ and $x^{(2)}$ are
independent. Thus,
\begin{align*}
g^\star(x,x') &=  \frac{\PP\left(x^{(1)}{=}x, x^{(2)}{=}x' \right)}{ \PP \left(x^{(1)}{=}x \right) \PP \left(x^{(2)}{=}x' \right)}\\
&= \sum_{k=1}^K \frac{\PP(w{=}e_k) \PP(x^{(1)}{=}x|w{=}e_k) \PP( x^{(2)}{=}x' |w{=}e_k)}{\PP \left(x^{(1)}{=}x \right) \PP \left(x^{(2)}{=}x' \right) }\\
&= \sum_{k=1}^K \frac{ \PP(w = e_k | x^{(1)} = x) \PP( x^{(2)} = x' |w = e_k)}{\PP \left(x^{(2)} = x' \right) } \\
 &= \frac{\eta(x)^\top \psi(x')}{\PP(x^{(2)} = x')},
\end{align*}
where the third equality follows from Bayes' rule.
\end{proof}

The above characterization shows that $g^\star$ contains information
about the posterior topic distribution $\eta(\cdot)$. To recover it,
we must make sure that the $\psi(\cdot)$ vectors for our landmark
documents span $\RR^K$. Formally, if $l_1, \ldots, l_M$ are the
landmarks, and we define the matrix $L \in \RR^{K\times M}$ by
\begin{equation} \label{eq:landmarkmatrix}
  L \ := \ 
\begin{bmatrix}
\frac{\psi(l_1)}{\PP(x^{(2)} = l_1)} & \cdots & \frac{\psi(l_M)}{\PP(x^{(2)} = l_M)}
\end{bmatrix},
\end{equation}
then our representation satisfies $g^\star(x, l_{1:M}) = L^\top
\eta(x)$. If our landmarks are chosen so that $L$ has rank $K$, then
there is a linear transformation of $g^\star(x, l_{1:M})$ that
recovers the posterior distribution of $w$ given $x$, i.e., $\eta(x)$. Formally, 
\begin{align*}
L^\dagger g^\star(x,l_{1:M}) = \eta(x)
\end{align*}
where $\dagger$ denotes the matrix pseudo-inverse.

There are two interesting observations here. The first is that this
argument naturally generalizes beyond the single topic setting to any
setting where $w$ can take values in a finite set $S$, which may
include some mixtures of multiple topics, though of course the number
of landmarks needed would grow at least linearly with $|S|$. The
second is that we have made no use of the structure of $x^{(1)}$ and
$x^{(2)}$, except for that they are independent conditioned on
$w$. Thus, this argument applies to more exotic ways of partitioning a
document beyond the bag-of-words approach.

\subsection{The general setting}

In the general setting, document vectors can be any probability vector
in $\Delta(\Kcal)$, and we do not hope to recover the
full posterior distribution over $\Delta(\Kcal)$. However, the
intuition from the single topic case largely carries over, and we are able
to recover the posterior moments.

Let $m_{\max}$ be the length of the longest landmark document. 
Let $S^K_m := \{ \alpha \in \ZZ^K_+ : \sum_{k} \alpha_k = m \}$ 
denote the set of non-negative integer vectors that sum to $m$ and let 
\[ S^K_{\leq m_{\max}} := \bigcup_{m=0}^{m_{\max}} S^K_m. \] 
Let $\pi(w)$ denote the degree-$m_{\max}$ monomial vector in $w$ as
\[ \pi(w) \ := \ \left( w_1^{\alpha_1} \cdots w_k^{\alpha_k}  :  \alpha \in S^K_{\leq m_{\max}} \right). \]
For a positive integer $m$ and a vector $\alpha \in S^K_m$, define the set
\[ {[m] \choose \alpha} := \left \{z \in [K]^m : \sum_{i=1}^m \ind[z_i = k] = \alpha_k  \ \ \ \forall k \in [K] \right \} .\]
Then for a document $x$ with length $m$, the degree-$m$ polynomial vector $\psi_m$ is defined by
\[ \psi_m(x) \ := \  \left( \sum_{z \in {[m] \choose \alpha}} \prod_{i =1}^m O(x_i | z_{i}) : \alpha \in S^K_m \right) \]
and let $\psi_d(x) = \vec{0}$ for all $d \neq m$. The cumulative polynomial vector $\psi$ is given by
\begin{equation} \label{eq:defnpsi}
\psi(x) \ := \ (\psi_0(x), \psi_1(x), \cdots, \psi_{m_{\max}}(x)).
\end{equation}

Given these definitions, we have the following general case analogue of Lemma~\ref{lem: single topic g}.
\begin{lemma}
\label{lem:general_case_g}
For any documents $x, x'$, 
\[ g^\star(x,x') = \frac{\eta(x)^\top \psi(x')}{\PP(x^{(2)} = x')}, \]
where $\eta(x) := \EE[ \pi(w) | x^{(1)} = x]$.
\end{lemma}
\begin{proof}[Proof sketch]
The proof is similar to that of Lemma~\ref{lem: single topic g},
albeit with more complicated definitions.
The key insight is that the probabability of a document
given topic factorizes as
\begin{align*}
\PP(x | w) 
 &= \sum_{z \in [K]^m} \left( \prod_{i=1}^m w_{z_i} \right) \left( \prod_{i=1}^m O(x_i | z_i) \right) \\
&=  \pi(w)^\top \psi(x) .
\end{align*}
From here, a similar derivation to Lemma~\ref{lem: single topic g}
applies. A full proof is deferred to the appendix.
\end{proof}
Therefore, we again have $g^\star(x, l_{1:M}) = L^\top \eta(x)$, but now the columns of $L$ correspond to vectors $\psi(l_i)$ from Eq.~\eqref{eq:defnpsi}.

When can we say something about the power of this representation? Our
analysis so far shows that if we choose the landmarks such that
$LL^\top$ is invertible, then our representation captures all of the
low-degree moments of the topic posterior. But how do we ensure that
$LL^\top$ is invertible? In the next theorem, we show that this is
possible whenever each topic has an associated \emph{anchor word},
i.e., a word that occurs with positive probability only within that
topic. In this case, there is a set of landmark documents $l_{1:M}$
such that any
polynomial of $\eta(x)$ can be expressed by a linear function of
$g^\star(x, l_{1:M})$.

\begin{theorem}
\label{thm: polynomial representation}
Suppose that
(i) each topic has an associated anchor word, and
(ii) the marginal distribution of $w$ has positive probability on some subset of the interior of $\Delta(\Kcal)$.
For any $d_o \geq 1$, there is a collection of $M = {{K+d_o} \choose d_o}$ landmark documents $l_1, \ldots, l_M$ such that if $\Pi(w)$ is a degree-$d_o$ polynomial in $w$, then there is a vector $\theta \in \RR^M$ such that
\[ \forall x: \langle \theta, g^\star(x, l_{1:M}) \rangle \ = \ \EE[ \Pi(w) | x^{(1)} = x].  \]
\end{theorem}
Combining Theorem~\ref{thm: polynomial representation} with the Stone-Weierstrass theorem~\cite{S48} shows that, in principle, we can approximate the posterior mean of any continuous function of the topic vector using our representation.

\begin{proof}[Proof of Theorem~\ref{thm: polynomial representation}]

By assumption (i), there exists an anchor word $a_k$ for each topic $k=1,\ldots, K$. By definition this means that $O(a_k | j) > 0$ if and only if $j = k$. For each vector $\alpha \in \ZZ_+^K$ such that $\sum \alpha_k \leq d_o$, create a landmark document consisting of $\alpha_k$ copies of $a_k$ for $k=1,\ldots, K$. This will result in ${{K+d_o} \choose d_o}$ landmark documents. Moreover, from assumption~(ii), we can see that each of these landmark documents has positive probability of occurring under the marginal distribution $\mu_2$, which implies $g^\star(x, l)$ is well-defined for all our landmark documents $l$.

Let $l$ denote one of our landmark documents and let $\alpha \in \ZZ_+^K$ be its associated vector. Since $l$ only contains anchor words, $\psi(l)_\beta > 0$ if and only if $\alpha = \beta$. To see this, note that
\[ \psi(l)_\alpha = \sum_{z \in {[m] \choose \alpha}} \prod_{i =1}^m O(l_i | z_{i}) \ \geq \  \prod_{k=1}^K O(a_k | k)^{\alpha_k} > 0.  \]
On the other hand, if $\beta \neq \alpha$ but $\sum_{k} \beta_k = \sum_{k} \alpha_k$, then there exists an index $k$ such that $\beta_k \geq \alpha_k + 1$. Thus, for any $z \in {[m] \choose \beta}$, there will be more than $\alpha_k$ words in $l$ assigned to topic $k$. Since every word in $l$ is an anchor word and at most $\alpha_k$ of them correspond to topic $k$, we will have
\[  \prod_{i =1}^m O(l_i | z_{i}) \ = \ 0. \]
Rebinding $\psi(l) = (\psi_0(l),\ldots,\psi_{d_0}(l))$ and
forming the matrix $L$ using this definition, we see that $L^\top$ can
be diagonalized and inverted.

For any target degree-$d_o$ polynomial $\Pi(w)$, there exists a vector
$v$ such that $\Pi(w) = \langle v, \pi_{d_0}(w) \rangle$, where
$\pi_{d_0}(w)$ denotes the degree-$d_0$ monomial vector. Thus, we may
take $\theta = {L}^{-1} v$ and get that for any document $x$:
\begin{align*}
\langle \theta, g^\star(x, l_{1:M}) \rangle \ &= \  ({L}^{-1} v)^T L^\top \eta(x)  \\
\ &= \  \EE[\langle v, \pi_{d_0}(w)\rangle | x^{(1)} = x] \\
\ &= \ \EE[ \Pi(w) | x^{(1)} = x]. \tag*\qedhere
\end{align*}
\end{proof}

\section{Error analysis}
\label{sec:errors}

Given a finite amount of data, we cannot hope to solve Eq.~\eqref{eqn: contrastive objective} exactly. Thus, our solution $\hat{f}$ will only be an approximation to $f^\star$. Since $\hat{f}$ is the basis of our representation, the fear is that the errors incurred in this approximation will cascade and cause our approximate representation $\phi(x)$ to differ so wildly from the ideal representation $g^\star(x, l_{1:M})$ that the results of Section~\ref{sec:topic-structure} do not even approximately hold.

In this section, we will show that, under certain conditions, such fears are unfounded. Specifically, we will show that there is an {error transformation} from the approximation error of $\hat{f}$ to the approximation error of linear functions in $\hat{\phi}$. That is, if the target function is $\eta(x)^\top\theta^\star$, then we will show that the risk of our approximate solution $\hat{\phi}$, given by
\begin{align*}
  R(\hat{\phi}) \defeq \min_{v} \EE_{x \sim \mu_1} (\eta(x)^\top\theta^\star - \hat{\phi}(x)^\top v)^2,
\end{align*}
is bounded in terms of the approximation quality of $\hat{f}$ as well as some other terms. Thus, for the specific setting of semi-supervised learning, an approximate solution to Eq.~\eqref{eqn: contrastive objective} is good enough.

It is worth pointing out that \citet{AKKPS19} also gave an error transformation from approximately solving a contrastive learning objective to downstream linear prediction. Also related, \citet{LSZ09} showed that when the approximation errors for their tasks are driven to zero, their representations will be perfect. However, they did not analyze what happens when their solutions have non-zero errors. In this sense, the results in this section are closer in spirit to those of~\citet{AKKPS19}.

In order to establish our error transformation, we first need to make some assumptions. Our first assumption is a consistency guarantee on our contrastive learning algorithm.

\begin{assum}
\label{assum: learning guarantee}
For any $\delta \in (0,1)$, there is a decreasing sequence $\varepsilon_n = o_n(1)$, such that given $n$ unlabeled documents the learning algorithm outputs a function $\hat{f}$ satisfying
\[ \EE_{(x, x') \sim \Dcal_c} \left[ \left( \hat{f}(x,x') - f^\star(x,x') \right)^2 \right] \leq \varepsilon_n \]
with probability $1-\delta$.
\end{assum}

If $\hat{f}$ is chosen from a bounded capacity function class $\Fcal$ by empirical risk minimization (ERM), Assumption~\ref{assum: learning guarantee} holds whenever $f^\star \in \Fcal$. Although this assumption is not essential to our analysis, it is needed to establish consistency in a semi-supervised learning setting.

There are a number of degrees of freedom for how to choose landmark documents. We consider a simple method: randomly sample them from the marginal distribution of $x^{(2)}$. Our next assumption is that this distribution satisfies certain regularity assumptions.

\begin{assum}
\label{assum: singular value}
There is a constant $\sigma_{\min} > 0$ such that for any $\delta \in (0,1)$, there is a number $M_0$ such that for an iid sample $l_1, \ldots, l_M$ with $M\geq M_0$, with probability $1-\delta$, the matrix $L$ defined in Eq.~\eqref{eq:landmarkmatrix} (with $\psi$ as defined in Eq.~\eqref{eq:defnpsi})
has minimum singular value at least $\sigma_{\min} \sqrt{M}$.
\end{assum}
Note that the smallest non-zero singular value of $\tfrac1{\sqrt{M}} L$ is the square-root of the smallest eigenvalue of an empirical second-moment matrix,
  \[ \frac1M \sum_{j=1}^M \frac1{\PP(\xb^{(2)} = \lb_j)^2} \psi(\lb_j) \psi(\lb_j)^\top . \]
Hence, Assumption~\ref{assum: singular value} holds under appropriate conditions on distribution over landmarks, for instance via tail bounds for sums of random matrices~\citep{Tropp12} combined with matrix perturbation analysis (e.g., Weyl's inequality). In the single topic setting with anchor words, it can be shown that for long enough documents, $\sigma_{\min}$ is lower-bounded by a constant for  $M_0$ growing polynomially with $K$. We defer a detailed proof of this to the appendix.

Our last assumption is that the predictions of $\hat{f}$ and $f^\star$ are non-negative and bounded below $1$. 

\begin{assum}
\label{assum: bounded values}
There exists a value $f_{\max} \in (0,1)$ such that for all documents $x$ and landmarks $l_i$
\[ 0 < \hat{f}(x,l_i), f^\star(x, l_i) \leq f_{\max}. \]
\end{assum}
Note that if Assumption~\ref{assum: bounded values} holds for $f^\star$, then it can be made to hold for $\hat{f}$ by thresholding. Moreover, it holds for $f^\star$ whenever the vocabulary and document sizes are constants, since we have for $\Delta = 1 - f^\star(x, x')$,
\begin{align*}
\Delta &=  \frac{\PP(x^{(1)} = x) \PP(x^{(2)} = x')}{\PP(x^{(1)} = x, x^{(2)} = x') + \PP(x^{(1)} = x) \PP(x^{(2)} = x')} \\
 & \geq \frac{ \PP(x^{(2)} = x')}{1+ \PP(x^{(2)} = x')}.
\end{align*}
Since the landmarks are sampled, and there are a finite number of possible documents, there exists a constant $p_{\min} >0$ such that $ \PP(x^{(2)} = l) \geq p_{\min}$. Thus, Assumption~\ref{assum: bounded values} holds for $ f_{\max} = {1}/({1+p_{\min}})$.

Given these assumptions, we have the following error transformation guarantee. The proof is deferred to the appendix.

\begin{theorem} 
\label{thm:error}
Fix any $\delta \in (0,1)$, and suppose Assumptions~\ref{assum: learning guarantee}-\ref{assum: bounded values} hold (with $M_0$, $\sigma_{\min}$, and $f_{\max}$). If $M \geq M_0$, there is a decreasing sequence $\varepsilon_n = o_n(1)$ such that with probability at least $1-\delta$ over the random sample of $l_1, \ldots l_M$ and the procedure for fitting $\hat{f}$,
 \begin{align*}
   R(\hat{\phi}) \leq \frac{\nbr{\theta^\star}_2^2}{\sigma_{\min}^2 (1-f_{\max})^4} \rbr{2 \varepsilon_n + \sqrt{\frac{2\log(3/\delta)}{M}}}.
\end{align*}
\end{theorem}

We make a few observations here. The first is that $\nbr{\theta^\star}_2^2$ is a measure of the complexity of the target function. Thus, if the target function is some reasonable function, say a low-degree polynomial, of the posterior document vector, then we would expect $\nbr{\theta^\star}_2^2$ to be small. The second is that the dependence on $f_{\max}$ is probably not very tight. Third, note that $n$ and $M$ are both allowed to grow with the amount of \emph{unlabeled} documents we have; indeed, none of the terms in Theorem~\ref{thm:error} deal with labeled data. 

Finally, if we have $n_L$ i.i.d.~labeled examples, and we learn a linear predictor $\hat{v}$ with the representation $\hat{\phi}$ using ERM (say), then the bias-variance decomposition grants
\begin{align*}
  \operatorname{mse}(\hat{v}) & \kern-1pt = \kern-1pt R(\hat{\phi}) + \kern-4pt \mathop{\EE}_{x \sim \mu_1} \kern-3pt ( \hat{\phi}(x)^\top( v^*{-}\hat{v}))^2 
=  R(\hat{\phi})  + O_P(\tfrac1{n_L})
\end{align*}
where $\operatorname{mse}(v) = \EE_{x \sim \mu_1} (\eta(x)^\top\theta^\star - \hat{\phi}(x)^\top v)^2$ and $v^*$ is the minimizer of $\operatorname{mse}(\cdot)$. The second equality comes from known properties of the ERM~\citep[see, e.g.,][]{HKZ14-ridge}.

\section{Semi-supervised experiments}
\label{sec:experiments}


We conducted experiments with our document level contrastive representations in a semi-supervised setting. In this section, we discuss the experimental details and findings. 

\subsection{A closely related representation}

One unfortunate consequence of the results in Section~\ref{sec:topic-structure} is that the number of landmarks required to obtain a useful representation can be quite large. To this end, we consider training models of the form $f_1, f_2: \Xcal \rightarrow \RR^d$ via
\begin{equation}
\label{eqn:bivariate_architecture}
\mini_{f_1, f_2} \EE_{\Dcal_{c}} \left[ \log \left(1+ \exp\left( - y  f_1(x)^\top f_2(x')    \right)  \right)\right].
\end{equation}
We will consider the alternate embedding scheme of simply taking $f_1(x)$ as our representation for document $x$. To justify this, first note that the Bayes optimal predictor $(f_1^\star,f_2^\star)$ is given by the log-odds ratio
\begin{align*}
f_1^\star(x)^\top f_2^\star(x') := \log\rbr{ \tfrac{\PP(y=1 \mid x,x')}{\PP(y=0 \mid x,x') }}.
\end{align*}
This predictor is related to our original $g^\star$ function via the
exponential:
\begin{align*}
g^\star(x,x') = \exp\left(f_1^\star(x)^\top f_2^\star(x')\right) \approx 1 + f_1^\star(x)^\top f_2^\star(x'),
\end{align*}
where the approximation comes from a Taylor expansion. Therefore,
if $l_1, \ldots, l_M$ are landmark documents, then $f^\star_1(x)$ is
approximately affinely related to $g^\star(x,l_{1:M})$:
\begin{align*}
g^\star(x,l_{1:M}) \ \approx \ \vec{1} + 
\begin{bmatrix}
  f_2^\star(l_1) & \dotsb & f_2^\star(l_M)
\end{bmatrix}^\top
f_1^\star(x).
\end{align*}
When 
the Taylor expansion is accurate, 
we can expect that the approximate minimizer $\hat{f}_1(x)$
of Eq.~\eqref{eqn:bivariate_architecture} is as good of a representation as
the version that uses landmarks. 


\subsection{Methodology}

We conducted semi-supervised experiments on the AG news topic classification dataset as compiled by~\citet{ZZL15}. This dataset contains news articles that belong to one of four categories: world, sports, business, and sci/tech. There are 30,000 examples from each class in the training set, and 1,900 examples from each class in the testing set. We minimally preprocessed the dataset by removing punctuation and words that occurred in fewer than 10 documents, resulting in a vocabulary of approximately 16,700 words. 

We randomly selected 1,000 examples from each class to remain as our labeled training dataset, and we used the remaining 116,000 examples as our unlabeled dataset for learning representations. After computing representations on the unlabeled dataset, we fit a linear classifier on the labeled training set using logistic regression with cross validation to choose the $\ell_2$ regularization parameter ($n_{\text{folds}} = 3$).


We compared our representation, \texttt{NCE}, against several representation baselines. 
\begin{itemize}
\item \texttt{BOW} -- The standard bag-of-words representation.
\item \texttt{BOW+SVD} -- A bag of words representation with dimensionality reduction. We first perform SVD on the bag-of-words representation using the unsupervised dataset to compute a low dimensional subspace, and train a linear classifier on the projected bag-of-words representations with the labeled dataset.
\item \texttt{LDA} -- A representation derived from LDA. We fit LDA on the unsupervised dataset using online variational Bayes~\cite{HBB10}, and our representation is the inferred posterior distribution over topics given training document. 
\item \texttt{word2vec} -- Skip-gram word embeddings~\cite{MSCCD13}. We fit the skip-gram word embeddings model on the unsupervised dataset and then averaged the word embeddings in each of the training documents to get their representation.
\end{itemize}

For our representation, to solve Eq.~\eqref{eqn:bivariate_architecture}, we considered neural network architectures of various depths. We used fully-connected layers with between 250 and 300 nodes per hidden layer. We used ReLU nonlinearities, dropout probability 1/2, batch normalization, and the default PyTorch initialization~\cite{PetAl19}. We optimized using RMSProp with momentum value 0.009 and weight decay 0.0001 as in~\citet{RBU19}. We started with learning rate $10^{-4}$ which we halved after 250 epochs, and we trained for 600 epochs.

To sample a contrastive dataset, we first randomly partitioned each unlabeled document in half to create the positive pairs. To create the negative pairs, we again randomly partitioned each unlabeled document in half, randomly permuted one set of half documents, and discarded collisions. This results in a contrastive dataset whose size is roughly twice the number of unlabeled documents. In the course of training our models for the contrastive task, we resampled a contrastive dataset every 3 epochs to prevent overfitting on any one particular dataset.

\subsection{Results}


Below we illustrate and discuss the results of our experiments. In all
line plots, the training examples axis refers to the number of
randomly selected labeled examples used to train the linear
classifier. The shaded regions denote 95\% confidence intervals
computed over 10 replicates of this random selection procedure.

\begin{figure*}[t]
\includegraphics[width=\textwidth]{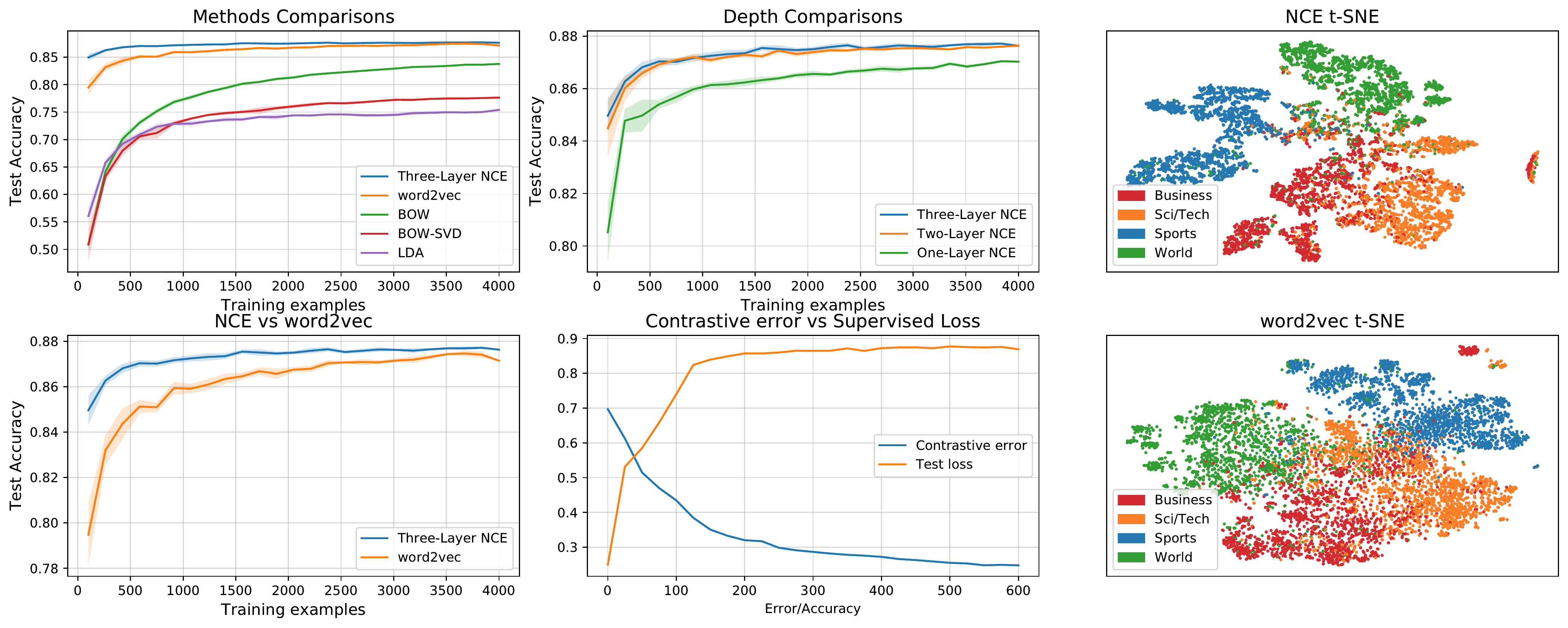}
\vspace{-1.5em}
\caption{Experiments with AG news dataset. Left panel: test accuracy
  of methods as we increase the number of supervised training
  examples. Bottom left focuses in on \texttt{NCE} versus
  \texttt{word2vec}. Top middle: \texttt{NCE} performance as we vary
  network depth. Bottom middle: Relationship between contrastive error
  and test accuracy for \texttt{NCE}. Right: t-SNE visualizations of
  \texttt{NCE} and \texttt{word2vec} embeddings.}
\label{fig:all_exp}
\vspace{-1em}
\end{figure*}

\paragraph{Baseline comparison.} 
We compared the semi-supervised perfomance of \texttt{NCE} against all
of the baselines. The left panel of Figure~\ref{fig:all_exp}
displays the results of these experiments. Among the methods tested,
\texttt{NCE} appears to outperform all the other methods, with
dramatic improvements over all methods except \texttt{word2vec} in the
low labeled data regime. Bag-of-words representations are quite
competitive when there is an abundance of labeled data, but as the
dimensionality of this representation is quite large, it has poor
performance with limited samples. However, unsupervised
dimensionality reduction on this representation appears to be 
unhelpful and actually degrades performance uniformly.

It is also worth noting that \texttt{LDA} performs quite poorly. This
could be for several reasons, including that fitting a topic model
directly could be challenging on the relatively short documents in the
corpus or that the document category is not well-expressed by a linear
function of the topic proportions.

Finally, we point out that word embedding representations
(\texttt{word2vec}) perform quite well, but our document-level
\texttt{NCE} procedure is slightly better, particularly when there are
few labeled examples. This may reflect some advantage in learning
document-level non-linear representations, as opposed to averaging
word-level ones.

\paragraph{Model capacity.}
We investigated the effect of depth on the performance of \texttt{NCE}
by training networks with one, two, and three hidden layers. In each
case, the first hidden layer has 300 nodes and the additional hidden
layers have 256 nodes. The top center panel of Figure~\ref{fig:all_exp} displays the
results. It appears that using deeper models in the unsupervised phase
leads to better performance when training a linear classifier
on the learned representations. We did not experiment exhaustively
with neural network architectures.

\paragraph{Contrastive loss.} 
We also tracked the contrastive loss of the model on a holdout
validation contrastive dataset. The bottom center panel of Figure~\ref{fig:all_exp}
plots how this loss evolves over training epochs. Along with this
contrastive loss, we checkpoint the model, train a linear classifier
on 1400 training examples, and evaluate the supervised test accuracy
as the representation improves. We see that test accuracy steadily
improves as contrastive loss decreases. This suggests that in these
settings, contrastive loss (which we can measure using an unlabeled validation
set) is a good surrogate for downstream performance (which may not be
measurable until we have a task at hand).



\paragraph{Visualizing embeddings.} 
For a qualitative perspective, we visualize the embeddings from
\texttt{NCE} using t-SNE with the default scikit-learn
parameters~\cite{MH08, scikit-learn}. To compare, we also used t-SNE
to visualize the document-averaged \texttt{word2vec}
embeddings. The right panels of Figure~\ref{fig:all_exp} shows these visualizations on the
7,600 test documents colored according to their true label. While
qualitiative, the visualization of the \texttt{NCE} embeddings appear
to be more clearly separated into label-homogeneous regions than that
of \texttt{word2vec}.

\section{Topic modeling simulations}
\label{sec:simulations}

The results of Section~\ref{sec:topic-structure} show that if a model is trained to minimize the contrastive learning objective, then that model must also recover certain topic posterior information in the corpus. However, there are a few practical questions that remain: can we train such a model, how much capacity should it have, and how much data is needed in order to train it? In this section, we present simulations designed to study these questions.

\subsection{Simulation setup}

We considered the following single topic generative model.

\begin{itemize}
	\item Draw topics $\theta_1, \ldots, \theta_K$ i.i.d. from a symmetric Dirichlet$(\alpha/K)$ distribution over $\Delta^{|\Vcal|}$.
	\item For each document:
	\begin{itemize}
		\item Draw a length $n \sim $ Poisson($\lambda$).
		\item Draw a topic $k \sim $ Uniform($[K]$).
		\item Draw $n$ words i.i.d. from $\theta_k$.
	\end{itemize}
\end{itemize}

This model can be thought of as a limiting case of the LDA model~\citep{BNJ03, GS04} when the document-level topic distribution is symmetric Dirichlet$(\beta)$ with $\beta \ll 1$. In our experiments, we set $K = 20$, $|\Vcal|=5000$, and $\lambda=30$, and we varied $\alpha$ from $1$ to 10. Notice that as $\alpha$ increases, the Dirichlet prior becomes more concentrated around the uniform distribution, so the topic distributions are more likely to be similar. Thus, we expect the contrastive learning problem to be more difficult with larger values of $\alpha$. 

We used contrastive models of the same form as Section~\ref{sec:experiments}, namely models of the form $f_1, f_2$ where the final prediction is $f_1(x)^\top f_2(x')$ and $f_1$ and $f_2$ are fully-connected neural networks with three hidden layers. To measure the effect of model capacity, we trained two models -- a smaller model with 256 nodes per hidden layer and a larger model with 512 nodes per hidden layer. Both models were trained for 100 epochs. We used all of the same optimization parameters as in Section~\ref{sec:experiments} with the exception of dropout, which we did not use.

To study the effect of training data, we varied the rate $r$ at which we resampled our entire contrastive training set from the ground truth topic model. Specifically, after every $1/r$-th training epoch, we resampled 60,000 new documents and constructed a contrastive dataset from these documents. We varied the resampling rate $r$ from 0.1 to 1.0, where larger values of $r$ imply more training data. The total amount of training data varies from 600K documents to 6M documents. 

Using the results from Section~\ref{sec:topic-structure}, we constructed the embedding $\phi(x)$ of a new document $x$ using 1000 landmark documents, each sampled from the same generative model. We constructed the true likelihood matrix $L$ of the landmark documents using the underlying topic model and recovered the model-based posterior $L^\dagger \phi(x)$. We measured accuracy as the fraction of testing documents for which the MAP topic under the model-based posterior matched the generating topic. We used 5000 testing documents and performed 5 replicates for each setting of parameters.

\subsection{Results}

Figure~\ref{fig:topic_exp} shows the results of our simulation
study. In the left panel, we plot the average pairwise topic
separation, measured in total variance distance, as a function of the
Dirichlet hyperparameter $\alpha$.  We see that, indeed as we increase
$\alpha$ the topics become more similar, which suggests that the
contrastive learning problem will become more difficult.  Then, in
the center and right panels we visualize the accuracy of the MAP estimates
on the test documents as a function of both the Dirichlet
hyperparameter $\alpha$ and the resampling rate $r$. The center panel
uses the small neural network with 256 nodes per hidden layer, while
the right panel uses the larger network.

The experiment identifies several interesting properties of the
contrastive learning approach. First, as a sanity check, the algorithm
does accurately predict the latent topics of the test documents in
most experimental conditions and the accuracy is quite high when the
problem is relatively easy (e.g., $\alpha$ is small). Second, the
performance degrades as $\alpha$ increases, but this can be mitigated
by increasing either the model capacity or the resampling
rate. Specifically, we consistently see that for a fixed model and
$\alpha$, increasing the resampling rate improves the accuracy. A
similar trend emerges when we fix $\alpha$ and rate and increase the
model capacity. These empirical findings suggests that latent
topics can be recovered by the contrastive learning approach, provided
we have an expressive enough model and enough data.


\begin{figure}[t]
\begin{center}
\includegraphics[width=0.95\textwidth]{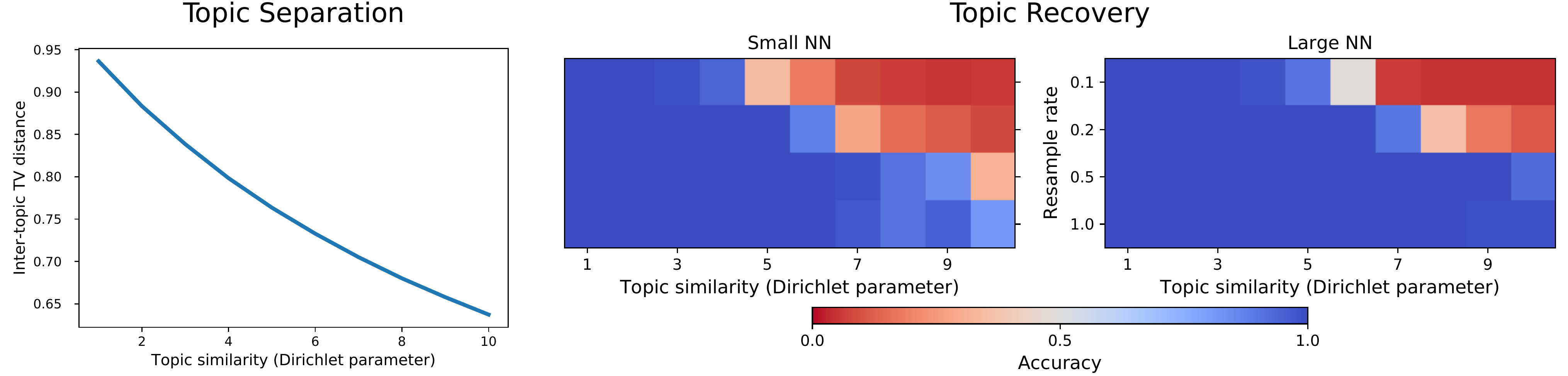}
\end{center}
\vspace{-1.5em}
\caption{Topic modeling simulations. Left: Average total variation distance between topics. Right: Topic recovery accuracy for contrastive models. Total number of documents sampled = 6M $\times$ rate.}
\label{fig:topic_exp}
\vspace{-1em}
\end{figure}


\section{Discussion}
Our analysis shows that
document-level contrastive learning
under topic modeling assumptions yields a
representation that exposes posterior topic information to linear
predictors, and hence is suitable for downstream
supervised learning.
In semi-supervised learning experiments, we show
that our contrastive learning procedure yields representations that
improve classification accuracy, and the improvement is most striking when we have few labeled examples.
We also explored the effects of model capacity and corpus size
in a simulated topic modeling study, and we showed that increasing 
either of these factors leads to higher quality topic recovery.

While we have focused on document representations and topic modeling
assumptions in this work, our analysis more generally sheds light on
the power of contrastive learning, which is empirically known to be useful in many
settings.
Aspects of our analysis may help
characterize the expressiveness of contrastive learning representations under other
modeling assumptions, for example in time-series
modeling, and we hope to
pursue these directions in future work.

\subsubsection*{Acknowledgements}

We thank Miro Dudík for initial discussions and suggesting the landmark embedding technique.
This work was partially completed while CT and DH were visiting Microsoft Research NYC, and was supported in part by NSF grant CCF-1740833.

\bibliography{refs}

\clearpage

\appendix
\section{Proofs}
\label{sec:proofs}

\subsection{Proof of general representation lemma}

\begin{proof}[Proof of Lemma~\ref{lem:general_case_g}]
Fix a document $x$ of length $m$ and a document probability vector
$w$. Conditioned on the assignment of each word in the document to a
topic, probability of a document factorizes as
\begin{align*}
\PP(x | w) &= \sum_{z \in [K]^m} \prod_{i=1}^m w_{z_i} O(x_i | z_i)
= \sum_{z \in [K]^m} \left( \prod_{i=1}^m w_{z_i} \right) \left( \prod_{i=1}^m O(x_i | z_i) \right)
=  \pi(w)^\top \psi(x),
\end{align*}
where the last line follows from collecting like terms. 
Using the form of $g^\star$ from above, we have
\begin{align*}
g^\star(x,x')  &= \frac{\PP(x^{(1)} = x, x^{(2)} = x')}{\PP(x^{(1)} = x) \PP(x^{(2)} = x')} 
= \frac{\int_w \PP(x^{(1)} = x | w) \PP( x^{(2)} = x' |w) \, d \PP(w) }{\PP(x^{(1)} = x) \PP(x^{(2)} = x')} \\
&= \frac{\int_w  \PP( x^{(2)} = x' |w) \, d \PP( w| x^{(1)} = x) }{\PP(x^{(2)} = x')} 
= \frac{\int_w \pi(w)^\top \psi(x) \, d \PP( w| x^{(1)} = x) }{\PP(x^{(2)} = x')} 
=  \frac{\eta(x)^\top \psi(x')}{\PP(x^{(2)} = x')}. \tag* \qedhere
\end{align*}
\end{proof}

\def\fmax{f_{\max}}

\subsection{Error analysis}
For the error analysis, recall that $\Dcal_c$ is our contrastive
distribution and 
\begin{align*}
f^\star(\xb,\xb') &\defeq \PP(y=1 \mid \xb,\xb'),\\
g^\star(\xb,\xb') &\defeq \frac{f^\star(\xb,\xb')}{1-f^\star(\xb,\xb')} = \frac{\PP(\xb^{(1)} = \xb, \xb^{(2)} = \xb')}{\PP(\xb^{(1)} = \xb)\PP(\xb^{(2)} = \xb')}, \\
\phi^\star(\xb) &\defeq g^\star(x, l_{1:M}) =  (g^\star(\xb,\lb_1),\ldots,g^\star(\xb,\lb_M))
\end{align*}
where $l_1, \ldots, l_M$ are landmark documents. Also recall our approximation $\hat{f}$ to $f^\star$, and the resulting approximations
\begin{align*}
\hat{g}(x,x') &\defeq \frac{\hat{f}(x,x')}{1 - \hat{f}(x,x')}, \\
\hat{\phi}(\xb) &\defeq (\hat{g}(\xb,\lb_1),\ldots, \hat{g}(\xb,\lb_M))
\end{align*}

Let $\eta(\xb), \psi(x)$ denote the posterior/likelihood vectors from Lemma~\ref{lem: single topic g} 
or the posterior/likelihood polynomial vectors from Lemma~\ref{lem:general_case_g}. Say
the length of this vector is $N \geq 1$.

Our goal is to show that linear functions in the representation $\hat{\phi}(\xb)$ can provide a good approximation to the target function
\begin{align*}
\xb \mapsto \eta(\xb)^\top \theta^\star
\end{align*}
where $\theta^\star \in \RR^{N}$ is some fixed vector. To this end, define the risk of $\hat{\phi}$ as
\begin{align*}
  R(\hat{\phi}) \defeq \min_{\vb} \EE_{\xb \sim \mu_1} (\eta(\xb)^\top\theta^\star - \hat{\phi}(\xb)^\top \vb)^2 .
\end{align*}

By~\pref{lem: single topic g} or \pref{lem:general_case_g}, we know that for any $\xb,\xb'$ we have
\begin{align*}
g^\star(\xb,\xb') = \frac{\eta(\xb)^\top\psi(\xb')}{\PP(\xb^{(2)} = \xb')}.
\end{align*}
Recall the matrix
\begin{align*}
L := \rbr{\frac{\psi(\lb_1)}{\PP(\xb^{(2)} = \lb_1)}, \ldots, \frac{\psi(\lb_M)}{\PP(\xb^{(2)} = \lb_M)}}.
\end{align*}
This matrix is in $\RR^{N \times M}$. If $L$ has full row rank, then
\begin{align*}
\eta(\xb)^\top\theta^\star = \eta(\xb)^\top L L^\dagger \theta^\star = \phi^\star(\xb)^\top \vb^\star
\end{align*}
where
\[ \phi^\star(\xb) := (g^\star(\xb,\lb_1),\ldots,g^\star(\xb,\lb_M)) \]
and $\vb^\star = L^\dagger \theta^\star$. Thus, $R(\phi^\star) = 0$. We will
show that $R(\hat{\phi})$ can be bounded as well.

\newtheorem*{theorem*}{Theorem}
\begin{theorem}
  \label{thm:error-general}
Suppose the following holds.
\begin{itemize}
  \item[(1)] There is a constant $\sigma_{\min} > 0$ such that for any $\delta \in (0,1)$, there is a number $M_0(\delta)$ such that for an iid sample $l_1, \ldots, l_M$ with $M\geq M_0(\delta)$, with probability $1-\delta$, the matrix
\[ L = 
\begin{bmatrix}
\frac{\psi(l_1)}{\PP(x^{(2)} = l_1)} & \cdots & \frac{\psi(l_M)}{\PP(x^{(2)} = l_M)}
\end{bmatrix} \]
has minimum singular value at least $\sigma_{\min} \sqrt{M}$.
	\item[(2)] There exists a value $f_{\max} \in (0,1)$ such that for all documents $x$ and landmarks $l_i$
\[ 0 < \hat{f}(x,l_i), f^\star(x, l_i) \leq f_{\max}. \]
\end{itemize}
    Let $\hat f$ be the function returned by the contrastive learning algorithm, and let
    \[ \varepsilon_n := \EE_{(x, x') \sim \Dcal_c} \left[ \left( \hat{f}(x,x') - f^\star(x,x') \right)^2 \right] \]
    denote its mean squared error.
    For any $\delta \in (0,1)$, if $M \geq M_0(\delta/2)$, then with probability at least $1-\delta$ over the random draw of $\lb_1,\ldots,\lb_M$, we have
  \begin{align*}
    R(\hat{\phi}) \leq \frac{\nbr{\theta^\star}_2^2}{\sigma_{\min}^2 (1-\fmax)^4} \rbr{2 \varepsilon_n + \sqrt{\frac{2\log(2/\delta)}{M}}}.
  \end{align*}
\end{theorem}
\begin{remark}
  \pref{thm:error} follows from this theorem by additionally conditioning on the event that $\hat f$ has the error bound in \pref{assum: learning guarantee}, and appropriately setting the failure probabilities $\delta$.
\end{remark}
\begin{proof}[Proof of Theorem~\ref{thm:error-general}]
  We first condition on two events based on the sample $\lb_1,\dotsc,\lb_M$.
  The first is the event that $L$ has full row rank and smallest non-zero singular value at least $\sqrt{M} \sigma_{\min} > 0$; this event has probability at least $1-\delta/2$.
  The second is the event that
  \begin{align}
    \frac1M \sum_{j=1}^M \EE_{\xb \sim \mu_1} \rbr{f^\star(\xb,\lb_j) - \hat{f}(\xb,\lb_j)}^2
    & \leq \EE_{(\xb,\xb')\sim\mu_1\otimes\mu_2}\rbr{f^\star(\xb,\lb_j) - \hat{f}(\xb,\lb_j)}^2 + \sqrt{\frac{2\log(2/\delta)}{M}} .
    \label{eq:hoeffding_event}
  \end{align}
  By Hoeffding's inequality and the assumption that $\hat{f}$ and $f^\star$ have range $[0,\fmax] \subseteq [0,1]$, this event also has probability at least $1-\delta/2$.
 By the union bound, both events hold simultaneously with probability at least $1-\delta$.We henceforth condition on these two events for the remainder of the proof.

Since $L$ has full row rank, via Cauchy-Schwarz, we have
\begin{align*}
  R(\hat{\phi})
  & = \min_{\vb} \EE_{\xb \sim \mu_1} (\eta(\xb)^\top\theta^\star - \hat{\phi}(\xb)^\top \vb)^2 
  \leq \EE_{\xb \sim \mu_1} (\eta(\xb)^\top \theta^\star - \hat{\phi}(\xb)^\top \vb^\star)^2\\
& = \EE_{\xb \sim \mu_1} ((\phi^\star(\xb)^\top - \hat{\phi}(\xb))^\top \vb^\star)^2
\leq \EE_{\xb\sim\mu_1} \nbr{\vb^\star}_2^2 \nbr{\phi^\star(\xb)^\top - \hat{\phi}(\xb)}_2^2 \\
& = \nbr{\vb^\star}_2^2 \cdot \EE_{\xb\sim\mu_1} \nbr{\phi^\star(\xb)^\top - \hat{\phi}(\xb)}_2^2.
\end{align*}
We analyze the two factors on the right-hand side separately.

\paragraph{Analysis of $\vb^\star$.}
For $\vb^\star$, we have
\begin{align*}
  \nbr{\vb^\star}_2^2 \leq \nbr{L^\dagger}_2^2 \nbr{\theta^\star}_2^2
  & \leq \frac{1}{M} \sigma_{\min}^2 \nbr{\theta^\star}_2^2 ,
\end{align*}
where we have used the fact that $L$ has smallest non-zero singular value at least $\sqrt{M}\sigma_{\min}$.

\paragraph{Analysis of $\phi^\star-\hat{\phi}$.}
For the other term, we have
\begin{align*}
\EE_{\xb \sim \mu_1} \nbr{\phi^\star(\xb) - \hat{\phi}(\xb)}_2^2 &= \sum_{j=1}^M \EE_{\xb \sim \mu_1} (g^\star(\xb,\lb_j) - \hat{g}(\xb,\lb_j))^2\\
& = \sum_{j=1}^M \EE_{\xb \sim \mu_1} \frac{\abr{f^\star(\xb,\lb_j) - \hat{f}(\xb,\lb_j)}^2}{(1-f^\star(\xb,\lb_j))^2(1-\hat{f}(\xb,\lb_j))^2}\\
& \leq \frac{1}{(1-\fmax)^4}\sum_{j=1}^M \EE_{\xb \sim \mu_1} \rbr{f^\star(\xb,\lb_j) - \hat{f}(\xb,\lb_j)}^2 \\
& \leq \frac{M}{(1-\fmax)^4}\rbr{ \EE_{(\xb,\xb')\sim\mu_1\otimes\mu_2}\rbr{f^\star(\xb,\xb') - \hat{f}(\xb,\xb')}^2 + \sqrt{\frac{2\log(2/\delta)}{M}} } ,
\end{align*}
where the final inequality follows from \eqref{eq:hoeffding_event}.


\paragraph{Wrapping up.}

Putting everything together, we have
\begin{align}
  R(\hat{\phi})
  & \leq \frac{\nbr{\theta^\star}_2^2}{\sigma_{\min}^2(1-\fmax)^4}
  \rbr{ \EE_{(\xb,\xb')\sim\mu_1\otimes\mu_2}\rbr{f^\star(\xb,\xb') - \hat{f}(\xb,\xb')}^2 + \sqrt{\frac{2\log(2/\delta)}{M}} } .
  \label{eq:penultimate_bound}
\end{align}
To conclude, we observe that half of the probability mass in $\Dcal_c$ is $\mu_1 \otimes \mu_2$, so
\[
  \varepsilon_n = \EE_{(\xb,\xb')\sim\Dcal_c}\rbr{f^\star(\xb,\xb') - \hat{f}(\xb,\xb')}^2
  \geq \frac12 \EE_{(\xb,\xb')\sim\mu_1\otimes\mu_2}\rbr{f^\star(\xb,\xb') - \hat{f}(\xb,\xb')}^2 .
\]
Rearranging and combining with \eqref{eq:penultimate_bound} proves the claim.
\end{proof}

\paragraph{Calculations about the minimum singular value.}
Suppose we are in the single topic case where $w \in
\{e_1,\ldots,e_K\}$. Assume that $\min_k \Pr(w=e_k)\geq w_{\min}$.
Further assumes that each topic $k$ has an anchor word $a_k$,
satisfying $O(a_k | z = e_k) \geq a_{\min}$. Then we will show that
when $M$ and $m$ are large enough, the matrix $L$ whose columns are
$\psi(x)/\PP(x)$ will have large singular values.

First note that if document $x$ contains $a_k$ then $\psi(x)$ is one
sparse, and satisfies
\begin{align*}
\textrm{if $a_k \in x$:}~~ \frac{\psi(x)}{\PP(x)} = \frac{e_k \PP(x | w=e_k)}{\sum_{k'} \PP(w=k')\PP(x | w=k')} = e_k/\PP(w=k')
\end{align*}
Therefore, the second moment matrix satisfies
\begin{align*}
\EE \frac{\psi(x)\psi(x)^\top}{\PP(x)^2} 
\succeq \sum_{k=1}^K \PP(w=e_k) \PP(a_k \in x \mid e_k) \EE\sbr{\frac{\psi(x)\psi(x)^\top}{\PP(x)} \mid a_k \in x, w=e_k}
 = \sum_{k=1}^K \frac{\PP(a_k \in x \mid e_k)}{\PP(w=e_k)} e_ke_k^\top
\end{align*}
Now, if the number of words per document is $m \geq 1/a_{\min}$ then
\begin{align*}
\PP(a_k \in x \mid e_k) = 1 - (1 - O(a_k \mid e_k))^{m} \geq 1-\exp(-m O(a_k|e_k)) \geq 1 - \exp(-ma_{\min}) \geq 1 - 1/e.
\end{align*}
Finally, using the fact that $\PP(w=e_k) \leq 1$, we see that the
second moment matrix satisfies
\begin{align*}
\EE\frac{\psi(x)\psi(x)^\top}{\PP(x)^2} \succeq (1-1/e) I_{K\times K}
\end{align*}
For the empirical matrix, we perform a crude analysis and apply the
Matrix-Hoeffding inequality. We have
$\nbr{\psi(x)\psi(x)^\top/\PP(x)^2}_2\leq Kw_{\min}^{-2}$ and so with probability at least $1-\delta$, we have
\begin{align*}
\nbr{\frac{1}{M}\sum_{i=1}^M \frac{\psi(l_i)\psi(l_i)^\top}{\PP(l_i)} - \EE\frac{\psi(x)\psi(x)^\top}{\PP(x)^2}}_2 \leq \sqrt{\frac{8 K\log(K/\delta)}{Mw_{\min}^2}}.
\end{align*}
If we take $M \geq \Omega(K\log(K/\delta)/w_{\min}^2)$ then we will
have that the minimum eigenvalue of the empirical second moment matrix
will be at least $1/2$.

\vfill

\end{document}